\documentclass{llncs}
\usepackage{pgfplots}
\pgfplotsset{plot coordinates/math parser=false}
\pgfplotsset{soldot/.style={color=blue,only marks,mark=*}} 
\pgfplotsset{holdot/.style={color=blue,fill=white,only marks,mark=*}}

\usepackage{graphicx} 
\usepackage{color}
\usepackage{thmtools,thm-restate}
\usepackage{url}

\usepackage{amsmath}
  \usepackage{algorithm}
\usepackage{algorithmic}
\usepackage{caption}
\usepackage{xspace}
\usepackage{multirow}
\newcommand{\prob}[1]{\ensuremath{\text{Pr}\left\{#1 \right\} }}
\newcommand{\fhm}{P-hype}
\newcommand{\fastia}{(1+1)~Fast-IA\xspace}
\newcommand{\oneonefastiaE}{(1+1)~Fast-IA$_{\geq}$}
\newcommand{\oneonefastiaG}{(1+1)~Fast-IA$_{>}$}
\newcommand{\fastoptia}{Fast Opt-IA}

\newcommand{\oneonerlsk}{(1+1)~RLS$_k$}

\newcommand{\oneoneiahype}{(1+1)~IA$^\text{hyp}$\xspace}
\newcommand{\hypfcm}{\fhm$_{FCM}$\xspace}
\newcommand{\hypnofcm}{\fhm$_{BM}$\xspace}
\newcommand{\onemax}{\textsc{OneMax}\xspace}
\newcommand{\leadingones}{\textsc{LeadingOnes}\xspace}

\newcommand\T{\rule{0pt}{2.6ex}}      
\newcommand\B{\rule[-1.2ex]{0pt}{0pt}}

\begin{document}
\mainmatter              % start of a contribution
\title{Fast Artificial Immune Systems}
% Fast Artificial Immune Systems via Stochastic Fitness Evaluations
\titlerunning{Fast AIS}  % abbreviated title (for running head)
\author{Dogan Corus \and Pietro S. Oliveto \and
Donya Yazdani}
%
%\authorrunning{Ivar Ekeland et al.} % abbreviated author list (for running head)
%Craig Chambers, Kim B. Bruce, and Elisa Bertino}
%
\institute{University of Sheffield, Rigorous Research, UK.\\
\email{d.corus@sheffield.ac.uk},
\email{p.oliveto@sheffield.ac.uk},
\email{dyazdani1@sheffield.ac.uk}
}

\maketitle              % typeset the title of the contribution

\begin{abstract}
%[at least 70 and at most 150 words] Notes for writing 
%Intro/Abstract: the new AIS helps in easy functions and also helps in escaping 
%from local optima and not going back to them. We do not waste many because we do 
%not evaluate all the time, but we do flip many bits. For $\gamma=1/(n \log^2 n)$ 
%the new operator evaluates the fitness 1 time in expectation, but flips many 
%bits which is equivalent to hypemutation without FCM. However, it is not going 
%to work on Cliff if $\gamma$ is big. 
%
Various studies have shown that characteristic Artificial Immune System (AIS) operators such as
hypermutations and ageing can be very efficient at escaping local optima of multimodal optimisation problems.
However, this efficiency comes at the expense of considerably slower runtimes during the exploitation phase compared to standard evolutionary algorithms. We propose modifications to the traditional `hypermutations with mutation potential' (HMP) that allow them to be efficient at exploitation as well 
as maintaining their effective explorative characteristics.  Rather than deterministically evaluating fitness after each bitflip of a hypermutation,  we sample the fitness function stochastically with a `parabolic' distribution which allows  the `stop at first constructive mutation' (FCM) variant of HMP to reduce the linear amount of wasted function evaluations when no improvement is found to a constant.  By returning the best sampled solution during the hypermutation, rather than the first constructive mutation, we then turn the extremely inefficient HMP operator without FCM, into a very effective operator for the standard Opt-IA AIS using hypermutation, cloning and ageing.
We rigorously prove the effectiveness of the two proposed operators by analysing them on all problems where the performance of HPM is rigorously understood in the literature.
\end{abstract}
%

% Dogan's to do list:
% \begin{itemize}
% \item To show for what $p_i$ the algorithm cannot solve Cliff
% \item Maybe we do not need FCM any more?
% \item The survival rate in stochastice(hybrid) ageing must either be $1-1/(\mu+1)$ or 
% $1-1/(2\mu)$, otherwise in (1+1)Opt-IA the ageing trick does not work.
% % %\item A lemma stating that if we flip more than n/2 bits, we end up in an 
% % area which is not interesting
% %\item Partition problem
% \end{itemize}

\section{Introduction}
%Inspired by the principles of vertebrates' natural immune system in recognising pathogens, optimisation Artificial Immune Systems (AIS) are developed for solving optimisation problems. 
%Some well-known AIS inspired by Burnet's clonal selection theory \cite{Burnet1959} are Clonalg \cite{DecastroVonzuben2002}, B-Cell \cite{KelseyTimmis2003} and Opt-IA \cite{CutelloTEVC}. These algorithms tipically use \textit{hyper}mutation operators (mutations with high rates) to create variation and use ageing operators to remove old solutions from the population that are not helping in exploring the search space. Designed very similar to the Evolutionary algorithms (EA), few has addressed this improtant question that on which problems AIS show advantage over EA with standard bit mutation (SBM). Among the rare research done with such a purpose, there are some works proving that B-Cell AIS outperforms EA equipped with SBM and crossover for optimising longest common subsequence \cite{JansenZarges2012} and vertex cover \cite{JansenOlivetoZarges2011} NP-hard problems. Recently it has also been shown that Opt-IA leads to considerable speed-ups in optimising toy problems (e.g., \textsc{Jump}, \textsc{Cliff}, \textsc{HiddenPath}, etc.) that EA equipped with SBM finds very difficult \cite{CorusOlivetoYazdani2017}. 

Several Artificial Immune Systems (AIS) inspired by Burnet's clonal selection principle~\cite{Burnet1959} have been developed to solve optimisation problems.
Amongst these, Clonalg \cite{DecastroVonzuben2002}, B-Cell \cite{KelseyTimmis2003} and Opt-IA \cite{CutelloTEVC,CutelloGecco03} are the most popular.
A common feature of these algorithms is their particularly high mutation rates compared to more traditional evolutionary algorithms (EAs).
For instance, the {\it contiguous somatic hypermutations} (CHM) used by the B-Cell algorithm, choose two random positions in the genotype of a candidate solution and 
flip all the bits in between\footnote{A parameter may be used to define the probability that each bit in the region actually flips. 
However, advantages of CHM over EAs have only been shown when all bits in the region flip.}. This operation results in a linear number of bits being flipped in an average mutation. 
The {\it hypermutations with mutation potential} (HMP) used by Opt-IA tend to flip a linear number of bits unless an improving solution is found first (i.e., if no {\it stop at first constructive mutation} mechanism (FCM) is used, then the operator fails to optimise efficiently any function with a polynomial number of optima~\cite{CorusOlivetoYazdani2017}).

Various studies have shown how these high mutation rates allow AIS to escape from local optima for which more traditional randomised search heuristics struggle.
%In particular, if many bits need to flip for an improvement to occur, then hypermutations may be very effective.
Jansen and Zarges proved for a benchmark function called  Concatenated Leading Ones Blocks (CLOB) an expected runtime of $O(n^2 \log n)$ using contiguous hypermutations versus the exponential time required by EAs relying on standard bit mutations (SBM) since many bits need to be flipped simultaneously to make progress~\cite{JansenZargesTCS2011}. 
Similar effects have also been shown on the NP-Hard longest common subsequence \cite{JansenZarges2012} and vertex cover \cite{JansenOlivetoZarges2011} standard combinatorial optimisation problems with practical applications where CHM efficiently escapes local optima where EAs (with and without crossover) are trapped for exponential time. 
%The B-Cell algorithm via CHM can efficiently optimise instances of the NP-Hard longest common subsequence \cite{JansenZarges2012} and vertex cover \cite{JansenOlivetoZarges2011} NP-hard problems by escaping local optima where EAs (with and without crossover) are trapped for exponential time.

This efficiency on multimodal problems comes at the expense of being 
considerably slower in the final exploitation phase of the optimisation process 
when few bits have to be flipped. For instance CHM requires  $\Theta(n^2 \log 
n)$ expected function evaluations to optimise the easy \textsc{OneMax} and 
\textsc{LeadingOnes}  benchmark functions. Indeed it has recently  been shown to 
require at least $\Omega(n^2)$ function evaluations to optimise any function 
since its expected runtime for its easiest function is 
$\Theta(n^2)$~\cite{EasiestFunctions}.
A disadvantage of CHM is that it is {\it biased}, in the sense that it behaves differently according to the order in which the information is encoded in the bitstring.
In this sense the unbiased HMP used by Opt-IA are easier to apply. % as their performance is independent of the particular encoding.
Also these hypermutations have been proven to be considerably efficient at escaping local optima such as those of the multimodal \textsc{Jump}, \textsc{Cliff}, and \textsc{Trap} benchmark functions that 
standard EAs find very difficult \cite{CorusOlivetoYazdani2017}. This performance also comes at the expense of being slower in the exploitation phase requiring, for instance,
 $\Theta(n^2 \log n)$ expected fitness evaluations for \textsc{OneMax} and $\Theta(n^3)$ for \textsc{LeadingOnes}.

In this paper we propose a modification to the HMP operator to allow it to be very efficient in the exploitation phases while  
maintaining its essential characteristics for escaping from local optima.
Rather than evaluating the fitness after each bit flip of a hypermutation as the traditional FCM requires, we propose to evaluate it based on the probability that the mutation will be successful.
The probability of hitting a specific point at Hamming distance $i$ from the current point, ${n \choose i}^{-1}$, decreases exponentially with the Hamming distance for $i < n/2$ and then it increases again in the same fashion. Based on this observation we evaluate each bit following a `parabolic' distribution such that the probability of evaluating the $i_{th}$ bit flip decreases as $i$ approaches $n/2$ and then increases again. We rigorously prove that the resulting hypermutation operator, which we call \hypfcm{}, locates local optima asymptotically as fast as Random Local Search (RLS) for any function where the expected runtime of RLS can be proven with the standard artificial fitness levels method. At the same time the operator is still exponentially faster than EAs for the standard multimodal 
\textsc{Jump}, \textsc{Cliff}, and \textsc{Trap} benchmark functions.

Hypermutations with mutation potential are usually applied in conjunction with Ageing operators in the standard Opt-IA AIS.
The power of ageing at escaping local optima has recently been enhanced by showing how it makes the difference between polynomial and exponential runtimes for the \textsc{Balance} function 
from dynamic optimisation~\cite{OlivetoSudholt2014}. For very difficult instances of \textsc{Cliff}, ageing even makes RLS asymptotically as fast as any unbiased mutation based algorithm can be on any function~\cite{Lehre2012} by running in $O(n \ln n)$ expected time~\cite{CorusOlivetoYazdani2017}.
However, the power of ageing at escaping local optima is lost when it is used in combination with hypermutations with mutation potential.
In particular, the  FCM mechanism does not allow the operator to accept solutions of lower quality, thus cancelling the advantages of ageing. Furthermore, the
high mutation rates combined with FCM make the algorithm return to the previous local optimum with very high probability. 
While the latter problem is naturally solved by our newly proposed \hypfcm{} that does not evaluate all bit flips in a hypermutation, the former problem requires a further modification to the HMP.
% such that Opt-IA may benefit of the advantages of both hypermutations and ageing. 
The simple modification that we propose is for the operator, which we call \hypnofcm, to return the best solution it has found if no constructive mutation is encountered.
We rigorously prove that Opt-IA then benefits from both operators  for all problems where it was previously analysed in the literature, as desired.

\section{Preliminaries}

%Introduced in \cite{CutelloTEVC} and studied rigorously in \cite{CorusOlivetoYazdani2017}, 
Static hypermutations with mutation potential using FCM (i.e., stop at the first constructive mutation) mutate $M=cn$ distinct bits for constant $0<c \leq 1$ 
and evaluate the fitness after each bitflip \cite{CorusOlivetoYazdani2017}. If an improvement over the original solution is found before the $M_{th}$ bitflip, then the operator stops and returns the improved solution. 
% The FCM operator is designed such that in iterations where an improvement is 
%found, the hypermutation stops flipping further bits and evaluating new  bitstrings. 
This behaviour prevents the hypermutation operator to waste further fitness function 
evaluations if an improvement has already been found. However, for any realistic objective function the number of 
iterations where there is an improvement constitutes an asymptotically small 
fraction of the total runtime. Hence, the fitness function evaluations saved due to 
the FCM stopping the hypermutation have a very small impact on the global performance 
of the algorithm. 
Our proposed modified hypermutation operator,  called \fhm, %uses a mutation potential $M=n$ which is also linear in the problem size. However, unlike the static hypermutation operator, it 
instead only evaluates the fitness after each bitflip with a probability that depends on how many bits have already been flipped in the current hypermutation operation. 
Since previous theoretical analyses have considered $c=1$ (i.e., $M=n$)~\cite{CorusOlivetoYazdani2017}, we also use this value throughout this paper.
Let $p_i$ be the probability that the solution is evaluated after the $i_{th}$ bit has been flipped. %Assuming $i$ as the number of the bit flip, $p_i$ 
The `parabolic' probability distribution is defined as follows, where the parameter $\gamma$ should be between~$0< \gamma \leq 2$: 
\begin{align}
\label{prob}
p_i= 
\begin{cases}
		1/e & \text{for}\; i=1 \;\text{and}\; i=n\\   
        \gamma/i & \text{for}\; 1<i\leq n/2\\
        \gamma/(n-i) & \text{for}\; n/2<i<n\\
\end{cases}
\end{align}

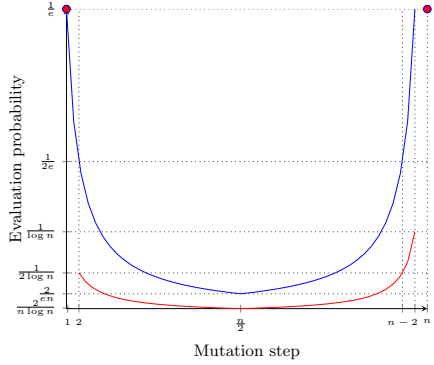
\begin{figure}[t]

\centering
\begin{tikzpicture}[ scale=0.7]
 \pgfmathsetmacro\ytkn{30} 
 \pgfmathsetmacro\ytknhalf{\ytkn/2} 
 \pgfmathsetmacro\ytkni{.22/\ytkn}
 \pgfmathsetmacro\ytknii{2/((2.7)*\ytkn)}
 \pgfmathsetmacro\ytkna{\ytkn-1} 
\pgfmathsetmacro\ytknb{\ytkn-2} 
 \pgfmathsetmacro\ytka{1/((2.7))}
 \pgfmathsetmacro\ytkb{1/(2*(2.7))}
 \pgfmathsetmacro\ytkc{1/(20)}
 \pgfmathsetmacro\ytkd{1/(10)}

\begin{axis}[
    axis x line=center, 
    axis y line=middle, 
      ytick={\ytkni, \ytknii, \ytka,\ytkb,  \ytkc, \ytkd},
yticklabels={$\frac{2}{n\log{n}}$, 
$\frac{2}{en}$,$\frac{1}{e}$,$\frac{1}{2e}$,$\frac{1}{2\log{ n } } $ , 
$\frac{1}{\log{n}}$, $$ ,$$ },
      xtick={1.1, 2,   \ytknb,\ytkna, \ytkn, \ytknhalf},
    xticklabels={1,2,$n-2$,$ $,$n$,$\frac{n}{2}$ },
    domain=1:\ytkn,
    ticklabel style = {font=\tiny},
      x label style={at={(axis description cs:0.5,-0.1)},anchor=north},
    y label style={at={(axis description cs:-0.1,.5)},rotate=90,anchor=south},
     xlabel={Mutation step},
    ylabel={Evaluation probability}]
]

% \draw (1,\ytka) node[anchor=south,huge, blue] {.};
\addplot[domain=1:\ytkn/2, blue] {1/((2.7)*x)};
\addplot[domain=\ytkn/2:\ytkn-1,blue] {1/(2.7*(\ytkn-x))};
\addplot[holdot, color=red] 
coordinates{(1,\ytka)(\ytkn,\ytka)};
\draw[dotted] (axis cs:1,\ytka) -- (axis cs:\ytkn,\ytka);
\draw[dotted] (axis cs:1,\ytkb) -- (axis cs:\ytkn,\ytkb);
\draw[dotted] (axis cs:1,\ytkc) -- (axis cs:\ytkn,\ytkc);
\draw[dotted] (axis cs:1,\ytknii) -- (axis cs:\ytkn,\ytknii);
\draw[dotted] (axis cs:1,\ytkd) -- (axis cs:\ytkn,\ytkd);
\draw[dotted] (axis cs:2,0) -- (axis cs:2,\ytka);
\draw[dotted] (axis cs:\ytknb,0) -- (axis cs:\ytknb,\ytka);
\draw[dotted] (axis cs:\ytkna,0) -- (axis cs:\ytkna,\ytka);
\draw[dotted] (axis cs:\ytkn,0) -- (axis cs:\ytkn,\ytka);
\addplot[holdot, color=blue,fill=red, ] 
coordinates{(1,\ytka  )(\ytkn,\ytka )};
\addplot[domain=2:\ytkn/2, red] {1/(10*x)};
\addplot[domain=\ytkn/2:\ytkn-1,red] {1/(10*(\ytkn-x))};

\end{axis}
\end{tikzpicture}
 \caption{The parabolic evaluation probabilities (\ref{prob}) for 
{\color{red}$\gamma=1/\log{n}$} and  {\color{blue}$\gamma=1/e$}.}
\end{figure}
% \begin{tikzpicture}
% \begin{axis}[ axis x line=center, 
%     axis y line=middle,  , domain]
% \addplot[domain=1:4 ,blue] {1/(2.7*x)};
% % % \addplot[domain=4:6,blue] {x};
% % % \addplot[domain=6:10,blue] {-5};
% % \draw[dotted] (axis cs:4,16) -- (axis cs:4,4);
% % \draw[dotted] (axis cs:6,6) -- (axis cs:6,-5);
% % \addplot[holdot] coordinates{(0,0)(4,4)(6,-5)};
% % \addplot[soldot] coordinates{(4,16)(6,6)(10,-5)};
% \end{axis}
% \end{tikzpicture}
% 

The lower the value of $\gamma$, the fewer the expected fitness function evaluations that occur in each hypermutation. 
For $\gamma=i$ we get the original static hypermutation. %Depending on using FCM, we define two versions of \fhm{} which are formally defined in Definition \ref{def:hyp-fcm} and Definition \ref{def:hyp-Nofcm}. 
On the other hand, with a small 
enough parameter $\gamma$ value, the number of wasted evaluations can be dropped to the 
order of $O(1)$ per iteration instead of the linear amount wasted by the traditional operator when improvements are not found. 
The resulting hypermutation operator is formally defined as follows.

\begin{definition}[\hypfcm{}]\label{def:hyp-fcm}
%The hypermutation operator 
 \hypfcm{} flips at most $n$ distinct bits selected uniformly at random. It evaluates the fitness after the $i_{th}$ bitflip with 
probability $p_i$ (as defined in (\ref{prob})) and remembers the last evaluation. \hypfcm{}  stops flipping bits when it finds an improvement; if no improvement is found, it will return the last evaluated solution. If no evaluations are made, the parent will be returned.
\end{definition}

In the next section we will prove its benefits over the standard static HMP with FCM, when incorporated into a (1+1) framework (Algorithm \ref{alg:fastia}). 
However, in order for the operator to work effectively in conjunction with ageing, a further modification is required.
%Instead of stopping at the first constructive mutation, the following proposed HMP variant performs all the bit flips and returns the best found improvement.
Instead of stopping the hypermutation at the first constructive mutation, we 
will execute all $n$ mutation steps, evaluate each bitstring with 
the probabilities in (\ref{prob}) and as the offspring, return the 
best solution evaluated during the hypermutation or the parent itself if 
no other bitstrings are evaluated. We will prove that such a modification, which 
we call \hypnofcm{}, may allow 
the complete Opt-IA to escape local optima more efficiently by \hypnofcm{} producing solutions of lower quality than the local optimum on which the algorithm was stuck
while individuals on the local optimum die due to ageing. \hypnofcm{} is formally defined as follows. 

\begin{definition}[\hypnofcm{}]\label{def:hyp-Nofcm}
\hypnofcm{} flips $n$ distinct bits selected uniformly at random. It evaluates 
the fitness after the $i_{th}$ bitflip with 
probability $p_i$ (as defined in (\ref{prob})) and remembers the best evaluation 
found so far. \hypnofcm{} returns the mutated solution with the best evaluation 
found. If no evaluations are made, the parent will be returned.
\end{definition}

For sufficiently small values of the parameter $\gamma$ only one function evaluation per hypermutation is performed in expectation (although all bits will be flipped).
Since it returns the best found one, this solution will be returned by \hypnofcm{} as it is the only one it has encountered.
Interestingly, this behaviour is similar to that of the HMP without FCM that also evaluates one point per hypermutation and returns it.
However, while HMP without FCM has exponential expected runtime for any function with a polynomial number of optima~\cite{CorusOlivetoYazdani2017}, we will show in the following sections that \hypnofcm{} can be very efficient. From this point of view, \hypnofcm{} is as a very effective way to perform hypermutations with mutation potential without FCM.

%The pseudo code of \fhm{} is shown in Algorithm \ref{alg:\fhm{}}. 
%We will  analyse the performance of this operator in isolation in a simple setting . 
In Section \ref{sec:fastoptia}, we consider \hypnofcm{} in the complete 
Opt-IA framework \cite{CutelloTEVC,CutelloGecco03,CorusOlivetoYazdani2017} hence analyse its performance 
combined with cloning and ageing.  
The algorithm which we call \fastoptia, is depicted in Algorithm 
\ref{alg:fastoptia}.
We will use the  \textit{hybrid ageing} operator as in \cite{CorusOlivetoYazdani2017,OlivetoSudholt2014}, which allows us to 
escape local optima. Hybrid ageing removes candidate solutions (i.e. b-cells) 
with probability $p_{die}=1-(1/(\mu+1))$ once they have passed an age threshold 
$\tau$. 
After initialising a population of $\mu$ b-cells with $age=0$, at each iteration the algorithm creates $dup$ copies of each b-cell. These copies are mutated by the \fhm{} operator, creating a population of mutants called $P^{hyp}$ which inherit the age of their parents if they do not improve the fitness; otherwise their age will be set to zero. At the next step, all b-cells with $age \geq \tau$ will be removed from both populations with probability $p_{die}$. 
If less than $\mu$ individuals have survived ageing, then the population is filled up with new randomly generated individuals.
At the selection phase, the best $\mu$ b-cells are chosen to form the population for the next generation.

\begin{algorithm}[t]
\caption{\fastia}
\begin{algorithmic}[1]
\STATE{Initialise $x$ uniformly at random.}
\WHILE{a global optimum is not found}
\STATE{Create $y=x$, then $y=\text{\fhm{}}(y)$;}
\STATE{If $f(y) \geq f(x)$, then $x=y$.}
\ENDWHILE
\end{algorithmic}
\label{alg:fastia}
\end{algorithm}

\begin{algorithm}[t]
\caption{\fastoptia}
\begin{algorithmic}[1]
\STATE{Initialise a population of $\mu$ b-cells, $P$, created uniformly 
at random;} 
\STATE{\textbf{for} each $x \in$ $P$ set $x^{age}=0$.}
\WHILE{a global optimum is not found}
\STATE{\textbf{for} each $x \in$ $P$ set $x^{age}=x^{age}+1$;}
\FOR{ $dup$ times for each $x \in P$ }
\STATE{$y=\text{\fhm{}}(x)$; }
\STATE{\textbf{if} $f(y)>f(x)$ \textbf{then } $y^{age}= 0$ \textbf{else} $y^{age}= x^{age}$;}
\STATE{Add $y$ to $P^{hyp}$.}
\ENDFOR
\STATE{Add $P^{hyp}$ to $P$, set $P^{hyp}=\emptyset$;}
\STATE{\textbf{for} each $x \in$ $P$ \textbf{if} $x^{age} \geq \tau$ \textbf{then} remove $x$ with probability $p_{die}$; }
\STATE{\textbf{if} $|P| < \mu $ \textbf{then } 
add $\mu-|P|$ solutions to $P$ with age zero generated uniformly at random;}
\STATE{\textbf{if} $|P| > \mu $ \textbf{then } remove $|P|-\mu$ solutions with the lowest fitness from $P$ breaking ties uniformly at random.}
%\STATE{Selection (P$^{main}$,P$^{hyp}$) without genotype diversity?}
\ENDWHILE
\end{algorithmic}
\label{alg:fastoptia}
\end{algorithm}
%
%The following lemma, Wald's Equation, is used to prove Lemma \ref{lem:wald}.
%
%\begin{lemma}[Wald's Equation \cite{MitzenmacheUpfal}]
%``Let $X_1,X_2,\cdots, X_i$ be non-negative, independent, identically 
%distributed random variables with distribution $X$. Let $T$ be the stopping time 
%for this sequence. If $T$ and $X$ have bounded expectation, then 
%$E[\sum_{i=1}^{T} X_i]=E[T] \cdot E[X]$." \end{lemma}
%

\section{Fast Hypermutations}

We start our analysis by relating the expected number of fitness function 
evaluations to the expected number of \fhm{} operations until the optimum 
is found. The following result holds for both \fhm{} operators.
The lemma quantifies the number of expected fitness function evaluations which are wasted by a hypermutation operation.
\begin{restatable}{lemma}{wald}\label{lem:wald}
 Let $T$ be the random variable denoting the number of \fhm{}
operations applied until the optimum is found. Then, the  expected number of 
total function evaluations is at most: $E[T]\cdot O(1+\gamma\log{n})$.
\end{restatable}
\begin{proof}
% Let the random variable $T$ denote the total number of \fhm{} operations that 
%will be executed until the optimal solution is sampled for the first time and
Let the random variable $X_i$  for $i \in [T]$ denote the number of fitness 
function evaluations during the $i$th execution of \fhm. Additionally, let 
the random variable $X_{i}'$ denote the number of fitness function evaluations 
at the $i_{th}$ operation assuming that no improvements are found.  For all $i$ it holds 
that  $X_i \preceq X_{i}'$ since finding an improvement can only decrease the 
number of evaluations. 
Thus, the total number of function evaluations $E[\sum_{i=1}^{T}X_i]$
can be bounded above by $E[\sum_{i=1}^{T}X_{i}']$ which is equal to $E[T]\cdot 
E[X']$ due to Wald's equation \cite{MitzenmacheUpfal} 
%(Lemma \ref{lem:wald})
 since $X_{i}'$ are identically distributed and 
independent from $T$.

We now write the expected number of fitness function evaluations in each operation as the 
sum of $n$ indicator variables $Y_i$ for $i\in[n]$ denoting whether an 
evaluation occurs after the $i$th bit mutation. Referring to the probabilities 
in (\ref{prob}), we get,
$
 E[X]=E\left[\sum\limits_{i=1}^{n}Y_i\right]=\sum\limits_{i=1}^{n} Pr\{Y_i=1\}
 = \frac{1}{e} + \frac{1}{e}+2\sum\limits_{i=2}^{n/2} \gamma \frac{1}{i} \leq 
\frac{2}{e} + 2\gamma\left(\ln{n/2}-1\right) $.
\qed 
\end{proof}

In Lemma~\ref{lem:wald}, $\gamma$ appears as a multiplicative factor 
in the expected runtime measured in fitness function evaluations. %Moreover, 
An intuitive lower bound of $\Omega(1/\log{n})$ for $\gamma$ can be inferred 
since smaller mutation rates will not decrease the runtime. While a smaller 
$\gamma$ does not decrease the asymptotic order of expected evaluations per 
operation, in Section~\ref{sec:fastoptia} we will provide an example where 
a smaller choice of $\gamma$ reduces $E[T]$ directly. For the rest of our 
results though, we will rely on $E[T]$ being the same as for the traditional static hypermutations with FCM
while the number of wasted fitness function evaluations decreases from $n$ to 
$O(1+\gamma\log{n})$.

We will now  analyse the simplest setting where we can implement \fhm{}. The  \fastia{} 
keeps 
a single individual in the population and uses \fhm{} to perturb it at every 
iteration. The performance of the \oneoneiahype{}, a similar barebones 
algorithm using the classical static hypermutation operator has recently been 
related to the performance of the well-studied Randomised Local Search 
algorithm (RLS) \cite{CorusOlivetoYazdani2017}. $\textsc{RLS}_k$ flips exactly $k$ bits of the current solution 
to sample a new search point, compares it with the current solution 
and continues with the new one unless it is worse. According to Theorem~3.3 and  Theorem~3.4 of 
\cite{CorusOlivetoYazdani2017}, any runtime upper bound for RLS obtained via 
Artificial Fitness Levels (AFL) method also holds for the \oneoneiahype{} with 
an additional factor of $n$ (e.g., an upper bound of $O(n)$ for RLS derived via AFL translates 
into an upper bound of $O(n^2)$ for the \oneoneiahype). The following theorem 
establishes a similar relationship between RLS and the \fastia{} with a factor of 
$O(1+\gamma\log{n})$ instead of $n$. In the context of the following theorem, 
\oneonefastiaE{} denotes the variant of \fastia{} which considers an equally good 
solution as constructive while \oneonefastiaG{} stops the hypermutation only if a 
solution strictly better than the parent is sampled.

\begin{restatable}{theorem}{aflk} \label{th:aflk}
Let $E\left(T^{AFL}_{A}\right)$ be any upper bound on the expected 
runtime of  algorithm A established by the artificial fitness levels method. 
Then\\ $E\left(T^{AFL}_{\text{\fastia}_>}\right) \leq 
E\left(T^{AFL}_{(1+1)~RLS_k}\right) \cdot k/\gamma \cdot O(1+\gamma\log{n})$. 
Moreover, for the special case of $k=1$, $E\left(T^{AFL}_{\text{\fastia}_{\geq}}\right) 
\leq E\left(T^{AFL}_{(1+1)~RLS_1}\right) \cdot O(1+\gamma\log{n})$ also holds. 
\end{restatable}

\begin{table}[t]
\caption{Expected runtimes of the standard (1+1)~EA and \oneoneiahype versus the expected runtime of the \fastia.
For $\gamma=O(1/\log{n})$, the \fastia is asymptotically at least as fast as the (1+1)~EA and faster by a linear factor compared to the \oneoneiahype  for the unimodal and trap functions.
For not too large jump and cliff sizes (i.e. $o(n/\log n)$), the \fastia has an asymptotic speed up compared to the \oneoneiahype for the same parameter setting.
For not too small jump/cliff sizes both AIS are much faster than the (1+1)~EA.
} 
% Theoretical results of Theorem~\ref{th:aflk}
%}
 \label{table:afl}
\begin{center}
\begin{tabular}{ |l l l l| }
 \hline
 Function&(1+1)~ EA&\oneoneiahype & \fastia
\T\B \\ \hline
 \textsc{OneMax}& $\Theta(n \log{n})$ \cite{DrosteJansenWegener2002} & $\Theta(n^2 \log{n})$ \cite{CorusOlivetoYazdani2017}& $\Theta\left(n 
\log{n}\left(1+\gamma \log{n}\right)\right)$
\T\B \\ 
 \textsc{LeadingOnes}& $\Theta(n^{2})$ \cite{DrosteJansenWegener2002} & $\Theta(n^3)$ \cite{CorusOlivetoYazdani2017}& $\Theta\left(n^2 
\left(1+\gamma \log{n}\right)\right)$ \T\B \\
 \textsc{Trap}& ${\Theta(n^n)}$ \cite{DrosteJansenWegener2002}& $\Theta(n^2 \log{n})$ \cite{CorusOlivetoYazdani2017} & $\Theta\left(n 
\log{n}\left(1+\gamma \log{n}\right)\right)$ \T\B \\
 $\textsc{Jump}_{d>1}$ & ${\Theta(n^d)}$ \cite{DrosteJansenWegener2002}& $O(n\binom{n}{d})$ \cite{CorusOlivetoYazdani2017} & 
$O\left(\left(d/\gamma\right) \cdot \left(1+\gamma 
\log{n}\right) \cdot \binom{n}{d}\right)$\T\B \\
 $\textsc{Cliff}_{d>1}$ & ${\Theta(n^d)}$ \cite{Jorge2015}& $O(n\binom{n}{d})$ \cite{CorusOlivetoYazdani2017} & 
$O\left(\left(d/\gamma\right) \cdot \left(1+\gamma 
\log{n}\right) \cdot \binom{n}{d}\right)$\T\B \\
\hline
 \end{tabular}
\end{center}
\end{table}

Apart from showing the efficiency of the \fastia, the theorem also allows easy achievements of upper bounds on the runtime of the algorithm, by just analysing the simple RLS. 
For $\gamma=O(1/\log{n})$, Theorem~\ref{th:aflk} implies the upper bounds of 
$O(n\log{n})$ and  $O(n^2)$ for classical benchmark functions \textsc{OneMax}  
and \textsc{LeadingOnes} respectively (see Table~\ref{table:afl}). Both of these 
bounds are asymptotically tight since each function's unary unbiased black-box 
complexity is in the same order as the presented upper bound \cite{Lehre2012}. 
%Regardless, in the following theorems we will provide more general lower bounds 
%which holds for any choice of $\gamma$.

% \begin{restatable}{theorem}{thonemax} \label{th:onemax}

\begin{corollary}\label{cor:onemax}
The expected runtimes of the \fastia{} to optimise 
\\$\textsc{OneMax}(x):=\sum_{i=1}^{n}x_i$ and 
$\textsc{LeadingOnes}:=\sum_{i=1}^{n}\prod_{j=1}^{i}x_j$ are 
respectively $O\left(n \log{n}\left(1+\gamma \log{n}\right)\right)$ and 
$O(n^2 \left(1+\gamma \log{n}\right))$. For $\gamma=O(1/\log{n})$ these 
bounds reduce to $\Theta(n \log{n})$ and $\Theta(n^2)$.
\end{corollary}
% \end{restatable}

% \begin{restatable}{theorem}{leadingones} \label{th:leadingones}
% The expected runtime of \fastia$_ \geq$ to optimise \textsc{LeadingOnes} is $\Theta(n^2 
% \left(1+\gamma \log{n}\right))$.
% \end{restatable}

\fhm{} samples the complementary bit-string with 
probability one if it cannot find any improvements. This behaviour allows an 
efficient optimisation of the deceptive \textsc{Trap} function which is 
identical to \textsc{OneMax} except that the optimum is in $0^n$. 
Since $n$ bits have to be flipped to reach the global optimum from the local optimum,  evolutionary algorithms based on standard bit mutation
require exponential runtime with overwhelming probability \cite{OlivetoYaoBookChapter}. By  
evaluating the sampled bitstrings stochastically, the \fastia{} provides up to a linear speed-up for small enough $\gamma$ compared to the  \oneoneiahype{}
on \textsc{Trap} as well. 

\begin{restatable}{theorem}{trap}
\label{thm:trap}
The expected runtime of the \fastia{} to optimise \textsc{Trap} is 
\\$\Theta(n \log{n} \left(1+\gamma \log{n}\right))$.
\end{restatable}

  The  results for the \oneoneiahype{} on \textsc{Jump} and \textsc{Cliff} 
functions \cite{CorusOlivetoYazdani2017} can also be adapted to the \fastia{}  in a straightforward manner, even 
though they fall out of the scope of Theorem~\ref{th:aflk}. Both $\textsc{Jump}_d$ 
and $\textsc{Cliff}_d$  have the same output as \textsc{OneMax} for bitstrings 
with up to $n-d$ 1-bits and the same optimum $1^n$. For solutions with the 
number of 1-bits between $n-d$ and $n$, \textsc{Jump} has a reversed 
\textsc{OneMax} slope creating a gradient towards $n-d$ while \textsc{Cliff} 
has a slope heading toward $1^n$ even though the fitness values are penalised by an additive factor $d$.
Being designed to accomplish larger mutations, the performance of hypermutations 
on \textsc{Jump} and \textsc{Cliff} functions is superior to standard bit 
mutation \cite{CorusOlivetoYazdani2017}. This advantage is preserved for the \fastia{} as seen in the following 
theorem.
% 
% \begin{align*}
% \textsc{Jump}_{(d)}(x):= \begin{cases}
% 		d+\sum_{i=1}^n x_i & \text{if}\; \sum_{i=1}^n x_i \leq n-d\\
%         &\text{or}\; \sum_{i=1}^n x_i=n \\
%         n-\sum_{i=1}^n x_i & \text{otherwise}
%  \end{cases}
% \end{align*}
% 
% \begin{align*}
% \textsc{Cliff}_{d}(x)=\begin{cases}
% 		\sum_{i=1}^n x_i & \text{if}\; \sum_{i=1}^n x_i \leq n-d \\
%         \sum_{i=1}^n x_i-d+ 1/2 & \text{otherwise}
%  \end{cases}
% \end{align*}

\begin{restatable}{theorem}{jump}
The expected runtime of the \fastia{} to optimise $\textsc{Jump}_d$ and 
$\textsc{Cliff}_d$ is $O\left(\left(d/\gamma\right) \cdot \left(1+\gamma 
\log{n}\right) \cdot \binom{n}{d}\right)$. \end{restatable}
%$d \cdot \log n \cdot \binom{n}{d}$
%\begin{corollary}
%The expected runtime of \fastia to optimise $\textsc{Cliff}_k$ is 
%$O(k \cdot \log n \cdot \binom{n}{d})$ for any $d$.
%\end{corollary}
For \textsc{Jump} and \textsc{Cliff}, the superiority of the \fastia{} in comparison to 
the deterministic evaluations scheme depends on the function parameter $d$. 
%With the assumption 
If $\gamma=\Omega(1/\log{n})$, the \fastia{} performs better 
for $d=o(n/\log{n})$ while the  deterministic scheme (i.e., \oneoneiahype)  is preferable for larger $d$. 
However, for small $d$ the difference between the runtimes can be as large as 
a factor of $n$ in favor of the \fastia{} while,  even for the largest $d$, the 
difference is less than a factor of $\log{n}$ in favor of the deterministic 
scheme. Here we should also note that for $d= \omega(n/\log{n})$ the expected 
time is exponentially large for both algorithms (albeit considerably smaller than that of standard EAs) and the $\log{n}$ factor has no 
realistic effect on the applicability of the algorithm.

\section{Fast Opt-IA}\label{sec:fastoptia}

In this section we will consider the effect of our proposed evaluation scheme on 
the complete Opt-IA algorithm. The distinguishing characteristic of the Opt-IA algorithm 
is its use of the ageing and hypermutation operators. In 
\cite{CorusOlivetoYazdani2017} a fitness function called \textsc{HiddenPath} (Fig. \ref{hp}) was presented where the use of 
both operators is necessary to find the optimum in polynomial time. The function 
\textsc{HiddenPath} provides a gradient to a local optimum, which allows the 
hypermutation operator to find another gradient which leads to the global 
optimum but situated on the opposite side of the search space (i.e., nearby the 
complementary bitstrings of the local optima). However, the ageing operator is 
necessary for the algorithm to accept worsening; otherwise the second 
gradient is not accessible. To prove our upper bound, we can follow the same proof strategy in 
\cite{archive}, which established an upper bound of $O(\tau \mu n+\mu n^{7/2})$ for the expected runtime of the traditional Opt-IA on \textsc{HiddenPath}.
We will see that Opt-IA benefits from an $n/\log{n}$ speedup due to \fhm. 	
\begin{figure}[t!]
 \centering
  \includegraphics[width=.6\textwidth]{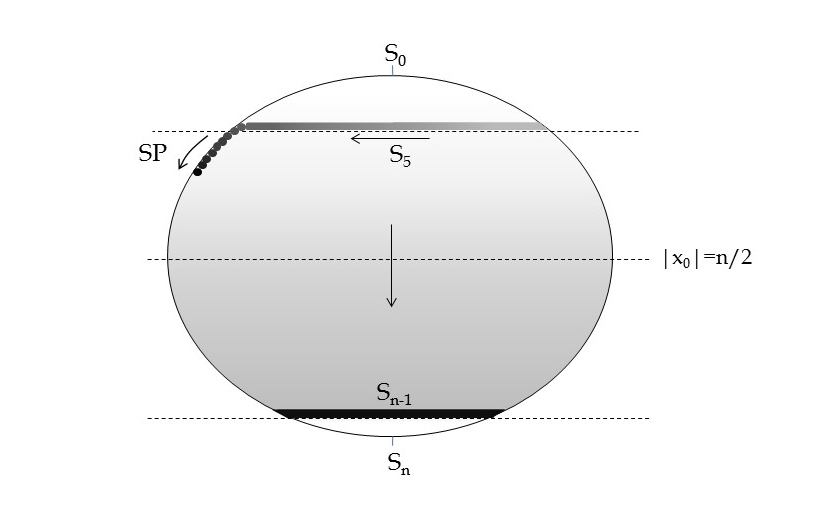}
 \caption{\textsc{HiddenPath} \cite{CorusOlivetoYazdani2017}}
 \label{hp}
 \end{figure}
 
\begin{restatable}{theorem}{hiddenpath}
The \fastoptia{} needs $O(\tau \mu + \mu n^{5/2}\log{n})$ fitness function 
evaluations  in expectation to optimise \textsc{HiddenPath} with
$\mu=O(\log n)$, $dup=1$, $1/(4\ln{n})\geq \gamma=\Omega(1/\log{n})$ and 
$\tau=\Omega(n \log^2 
n)$. 
\end{restatable}

\textsc{HiddenPath} was artificially constructed to fit the behaviour of the 
Opt-IA to illustrate its strengths. One of those strengths was the 
ageing mechanism's ability to escape local optima in two different ways. First, it allows the algorithm to restart with a new random population after it gets 
stuck at a local optimum. Second, ageing allows individuals 
with worse fitness than the current best to stay in the population when all the current best individuals 
are removed by the ageing operator in the same iteration. If an improvement is found soon 
after the worsening is accepted, then this temporary non-elitist behaviour 
allows the algorithm to follow other gradients which are accessible by variation 
from the local optima but leads away from them. On the other hand, even though 
it is coupled with ageing in the Opt-IA, the FCM mechanism does not allow 
worsenings. More precisely, for the hypermutation with FCM, the complementary 
bit-string of the local optimum is sampled with probability $1$ if no other 
improvements are found. 
Indeed, \textsc{HiddenPath} was designed to exploit 
this high probability. 
However, by only stopping on improving mutations, the traditional hypermutations with FCM do not allow, in general, to take advantage of the power of ageing at escaping local optima.
For instance, for the classical benchmark function 
$\textsc{Cliff}_d$ with parameter $d=\Theta(n)$, hypermutation with FCM turned 
out to be a worse choice of variation operator to couple with ageing than both 
local search and standard-bit-mutation \cite{archive}. Ageing coupled with RLS 
and 
SBM can reach the optimum by local moves, which respectively yields upper bounds of  
$O(n\log{n})$ and $O(n^{1+\epsilon}\log{n})$ for arbitrarily small constant 
$\epsilon$ on their runtimes. However, hypermutations with FCM require to 
increase the number of 1-bits in the current solution by $d$ at least once 
before the hypermutation stops. This requirement implies the following exponential lower bound on the runtime 
regardless of the evaluation scheme (as long as the hypermutation only stops on a constructive mutation). 

\begin{restatable}{theorem}{fcmcliff} \label{th:fcmcliff}
\fastoptia{} using \hypfcm{}  requires at least $2^{\Omega(n)}$ fitness function evaluations in expectation to find the optimum of 
$\textsc{Cliff}_d$ for $d=(1-c)n/4$, where $c$ is a constant $1>c>0$.
\end{restatable}

% The FCM operator is designed such that in iterations where an improvement is 
%found, the hypermutation stops flipping further bits and evaluating new 
%bitstrings. This behaviour prevents the hypermutation operator to waste further fitness function 
%evaluations. However, for any realistic objective function the number of 
%iterations where there is an improvement constitutes an asymptotically small 
%fraction of the total runtime. Hence, the fitness function evaluations saved due to 
%the FCM stopping the hypermutation have a very small impact on the performance 
%of the algorithm. Moreover, with the stochastic evaluation scheme and a small 
%enough parameter $\gamma$ the number of wasted evaluations can be dropped to the 
%order of $O(1)$ per iteration. Hence we propose \hypnofcm{} with a new 
%mechanism compatible with the ageing mechanism's ability to escape from local 
%optimum to replace the FCM mechanism in \hypfcm{} (Definition \ref{def:hyp-Nofcm}). 
%Instead of stopping the hypermutation at the first constructive mutation we 
%will execute all $n$ mutation steps, evaluate each bitstring with 
%the probabilities in (\ref{prob}) and as the offspring, return the 
%best solution evaluated during the hypermutation or the parent itself if no 
%other evaluated bitstring has a better fitness. Next we 
The following theorem will demonstrate how \hypnofcm{}, 
that, instead of stopping the hypermutation at the first constructive mutation, 
will execute all $n$ mutation steps, evaluates each bitstring with the probabilities in (\ref{prob}) and return the best found solution,
allows ageing and hypermutation to work in harmony in Opt-IA.

\begin{restatable}{theorem}{nfcmpos}\label{thm:cliff}
\fastoptia{} using \hypnofcm{} with $\mu=1$, $dup=1$, $\gamma=1/(n \log^2 n)$ and 
$\tau= \Theta(n \log n)$  needs $O(n \log n)$ fitness function evaluations in 
expectation to optimise \textsc{Cliff} with any linear $d \leq n/4-\epsilon$ for 
an small constant $\epsilon$. \end{restatable}

Note that the above result requires a $\gamma$ in the order of 
$\Theta(1/(n\log^2{n}))$, while Lemma~\ref{lem:wald} implies that 
any $\gamma=\omega(1/\log{n})$ would not decrease the expected number of 
fitness function evaluations below the asymptotic order of $\Theta(1)$. 
However, having $\gamma=1/(n\log^2{n})$ allows Opt-IA with constant 
probability to complete its local search before any solution with larger 
Hamming distance is ever evaluated. In  Theorem~\ref{thm:cliff}, we 
observe that this opportunity allows the Opt-IA to hillclimb the second slope 
before jumping back to the local optima. 
%Although we determine the necessary choice 
%for the parameter $\gamma$ precisely, we can rigorously prove that 
The following theorem rigorously proves that a very small choice for $\gamma$ in this case is necessary (i.e., 
$\gamma=\Omega(1/\log{n})$ leads to exponential expected runtime). 

\begin{restatable}{theorem}{nfcmneg}\label{th:nfcmneg}
At least $2^{\Omega(n)}$ fitness function evaluations in expectation 
are executed 
before  the \fastoptia{} using \hypnofcm{} with $\gamma=\Omega(1/\log{n})$  finds the 
optimum of $\textsc{Cliff}_d$ for $d=(1-c)n/4$, where $c$ is a constant $1>c>0$.
\end{restatable}

% 
% \begin{align*}
% \textsc{Simple Trap}(x):=
% \begin{cases} 
% 	\frac{a}{z}(z-\textsc{OneMax}(x)) & \text{if} \; \textsc{OneMax}(x)\leq z\\
%     \frac{b}{n-z}(\textsc{OneMax}(x)-z) & \text{otherwise}
% \end{cases}
% \end{align*}
% where $z \approx n/4$, $b=n-z-1$ and  $3b/2 \leq a \leq 2b$ and the optimal solution is the $0^n$ bit string.
% 
% 
% 
% 
% \begin{align*}
% \textsc{HiddenPath}(x)= 
% \begin{cases}
% 		n-\epsilon + \frac{\sum_{i=n-4}^n (1-x_i)}{n}  & \text{if}\; \sum_{i=1}^n (1-x_i)=5 \text{ and } x \neq 1^{n-5}0^5\\
%         0 & \text{if}\; \sum_{i=1}^n (1-x_i)<5 \text{ or } \sum_{i=1}^n (1-x_i)=n\\
%        n-\epsilon+\epsilon k/\log n & \text{if } 5 \leq k \leq \log{n}+1\; and 
% \; x=1^{n-k}0^k  \\
%        n & \text{if}\; \sum_{i=1}^n (1-x_i)=n-1\\
%      \sum_{i=1}^n (1-x_i) & \text{otherwise}
%  \end{cases}
%  \label{hiddenpath}
% \end{align*}
% 
% 

\section{Conclusion}
Due to recent analyses of increasingly realistic evolutionary algorithms,
higher mutation rates, naturally present in artificial immune systems, than 
previously recommended or used as a rule of thumb, are gaining significant 
interest in the evolutionary computation community~\cite{OlivetoLehreNeumann2009,Doerretal2017,corus2017standard,OlivetoTEVC2017}.

We have presented two alternative `hypermutations with mutation potential' operators, \hypfcm and \hypnofcm and have rigorously proved, %that they
%outperform the traditional HMP, respectively, with and without FCM 
for several significant benchmark problems from the literature, that they
%In particular, the proposed operators 
maintain the exploration characteristics of the traditional operators while outperforming them up to linear factor speedups in the exploitation phase.

The main modification that allows to achieve the presented improvements is to sample the solution after the $i_{th}$ bitflip stochastically with probability roughly $p_i = \gamma/i$, rather than deterministically
with probability one. The analysis shows that the parameter $\gamma$ can be set easily. 
Concerning \hypfcm, that returns the first sampled constructive mutation and is suggested to be used in isolation, any $\gamma =O(1/\log(n))$ 
allows optimal asymptotical exploitation time (based on the unary unbiased black box complexity of \onemax and \leadingones) while maintaining the traditional exploration capabilities.
Concerning \hypnofcm, which does not use FCM and is designed to work harmonically with ageing as in the standard Opt-IA, considerably lower values of the parameter (i.e., $\gamma= 1/(n\log^2{n})$) are required to escape from difficult local optima efficiently (eg. \textsc{Cliff}) such that the hypermutations do not return to the local optima with high probability. While these low values for $\gamma$ still allow optimal asymptotic exploitation in the unbiased unary black box sense, 
they considerably reduce the capability of the operator to perform the large jumps required to escape the local optima of functions with characteristics similar to \textsc{Jump}, i.e., where ageing
is ineffective due to the second slope of decreasing fitness.
Future work may consider an adaptation of the parameter $\gamma$ to allow it to automatically  increase and decrease throughout the run~\cite{DLOW2018,DoerrDoerr2018}.
Furthermore, the performance of the proposed operators should be evaluated for classical combinatorial optimisation problems and real-world applications.

\bibliographystyle{unsrt}
\bibliography{mybib2} 

\newpage
\appendix
\section{Appendix}
This appendix contains additional material to be read at the discretion of the reviewers. It is not necessary to understand the main part of the paper (the first 12 pages, which is the paper we submit for a possible publication in the conference proceedings), but it allows to check the correctness of the mathematical results. This appendix thus has a similar role as providing the source code in an experimental publication. If this submission is accepted for PPSN, we will publish a preprint containing all proofs, so that also the readers of the proceedings have access to this material.

The benchmark functions analysed in the paper are formally defined as follows:

 \begin{align*}
\textsc{Jump}_{d}(x):= \begin{cases}
		d+\sum_{i=1}^n x_i & \text{if}\; \sum_{i=1}^n x_i \leq n-d 
\;\text{or}\; \sum_{i=1}^n x_i=n \\
        n-\sum_{i=1}^n x_i & \text{otherwise}.
 \end{cases}
 \end{align*}

\begin{align*}
\textsc{Cliff}_{d}(x)=\begin{cases}
		\sum_{i=1}^n x_i & \text{if}\; \sum_{i=1}^n x_i \leq n-d \\
        \sum_{i=1}^n x_i-d+ 1/2 & \text{otherwise}.
 \end{cases}
\end{align*}

\begin{align*}
\textsc{HiddenPath}(x)= 
\begin{cases}
		n-\epsilon + \frac{\sum_{i=n-4}^n (1-x_i)}{n}  & \text{if}\; 
\sum_{i=1}^n (1-x_i)=5 \text{ and } x \neq 1^{n-5}0^5\\
        0 & \text{if}\; \sum_{i=1}^n (1-x_i)<5 \\
0 & \text{if}\;  \sum_{i=1}^n 
(1-x_i)=n\\
       n-\epsilon+\epsilon k/\log n & \text{if } 5 \leq k \leq \log{n}+1\; \text{and} 
\; x=1^{n-k}0^k  \\
       n & \text{if}\; \sum_{i=1}^n (1-x_i)=n-1\\
     \sum_{i=1}^n (1-x_i) & \text{otherwise}.
 \end{cases}
 \label{hiddenpath}
\end{align*}

We made use of the following theorem by Serfling which provides an 
upper bound on the outcome of a hypergeometric distribution. Consider a set 
$C:=\{c_1, \ldots, c_n\}$ consisting of $n$ elements, 
with $c_i \in R$ where $c_{min}$ and $c_{max}$ are the smallest and largest 
elements in $C$ respectively. Let $\bar{\mu}:= (1/n)\sum_{j=1}^{n}c_i$, be the 
mean of $C$. Let $1\leq i \leq k \leq n$ and $X_i$ denote the $i$th 
draw without replacement from $C$ and $\bar{X}:=(1/k)\sum_{j=1}^{k} X_i$  the 
sample mean.
\begin{theorem}[Serfling \cite{serfling1974probability}]\label{thm:serfling}
 For $1\leq k \leq n$, and $\lambda>0$
 \[Pr\left\{  \sqrt{k}  (\bar{X} - \bar{\mu}) \geq \lambda \right\}\leq 
\exp\left(-\frac{2\lambda^2}{(1-f^{*}_{k})(c_{max}-c_{min})^{2}}\right)\] 
where $f^{*}_{k}:= \frac{k-1}{n}$. 
\end{theorem}

%\wald*

Artificial Fitness Levels (AFL) is a method to derive upperbound on the expected 
runtime of a (1+1) algorithm \cite{AugerDoerrBook}. AFL divides the search space into $m$ 
mutually exclusive partitions $A_1 \cdots, A_m$ such that all the points in 
$A_i$ have less fitness than points which belong to $A_j$ for all $j>i$. The 
last partition, $A_m$ only includes the global optimum. If $p_i$ is the 
smallest probability that an individual belonging to $A_i$ mutates to an 
individual belonging to $A_j$ such that $i<j$, then the expected time to find 
the optimum is $E(T) \leq \sum_{i=1}^{m-1}1/p_i$. This method is used in the below theorem.
\aflk*

\begin{proof}
The \fastia{} selects which $k$ bits will be flipped first with the same 
distribution as \oneonerlsk{} (uniformly at random without replacement) and 
evaluates the solution with probability $\gamma/k$ if $k<n/2$ and with a 
probability greater than $\gamma/k$ otherwise. Thus, 
from any initial solution, the
\fastia{} can improve at least with the same probability as \oneonerlsk{} 
multiplied by $\gamma/k$ if it is not stopped before the $k$th mutation step. 
With strict selection, it is guaranteed that the $k$th bitflip will happen at 
each iteration unless an improvement occurs. For the $k=1$ case, the 
hypermutation cannot be stopped by a prior mutation step and the evaluation 
probability is always $1/e$. Pessimistically, we 
assume that if the exact $k$ bits are not flipped, then the algorithm will not improve 
and all the fitness evaluations of the current mutation are wasted. The factor $O(1+\gamma \log n)$ comes from Lemma~\ref{lem:wald}. \qed
\end{proof}

\jump*
\begin{proof}
According to Corollary~\ref{cor:onemax}, the time to sample a solution with $n-d$ 1-bits is at most $O\left(n\log{n}\left(1+\gamma \log{n}\right)\right)$ because the function behaves as \textsc{OneMax} for solutions with less than $n-d$ 1-bits. The Hamming distance of locally optimal points to the global 
optimum is $d$, thus, the probability of reaching the global optimum at the $d_{th}$
mutation step is $\binom{n}{d}^{-1}$ while the probability of evaluating is $\gamma/d$. Using Lemma~\ref{lem:wald}, we bound the total expected time to optimise \textsc{Jump} and \textsc{Cliff} as $E(T) \leq O( d/\gamma ) \cdot O(1+\gamma \log n) \cdot \binom{n}{d}$. \qed
\end{proof}

\trap*
\begin{proof}
% We follow the same arguments as in the proof of Theorem 5.5 in 
% \cite{CorusOlivetoYazdani2017} for Opt-IA using the static hypermutation 
% operator for optimising the simple trap function. 
% % Here 
% We consider the first mutation step so the probability of improving and 
% evaluating in each step is at least $ i/n \cdot 1/e$ with $i$ indicating the 
% number of 0-bits. If the ageing mechanism does not get triggered until the 
% optimum is found, 
% We can adapt Theorem \ref{th:onemax} for \fastoptia. Hence, 
% The solution quality improves as the number of bit-strings gets 
% further away from $z$. 
% Thus, 
According to Corollary~\ref{cor:onemax} we can 
conclude that the current individual will reach $1^n$ in $ O(n \log n 
\cdot(1+\gamma \log n))$ steps in expectation. 
% Pessimistically assuming that $1^n$ is found first, 
The global optimum is found in a single step with probability $1/e$ by 
evaluating after flipping every bit for which the number of additional fitness 
evaluations is $O(1+\gamma \log n)$ in expectation. 
% Using Markov's inequality iteratively, w.o.p the optimum is found within 
% $\tau=\Omega(n^2)$ steps. 
\qed 
\end{proof}

\hiddenpath*

\begin{proof}
We follow similar arguments to those of the proof of Theorem 11 in 
\cite{archive} for Opt-IA optimising \textsc{HiddenPath}. 
%In that proof, the probabilities of events are calculated according to generations until the very end when the expected runtime is stated in fitness function evaluations. We also use generations in all of our arguments here, hence the proof is mainly similar to that of Opt-IA.
 During the analysis, we call a non-SP solution which has $i$ 0-bits an $S_i$ solution for simplicity. 

An $S_{n-1}$ solution will be found in $1/e \cdot O(n \log n)$ 
generations by hill-climbing the \textsc{ZeroMax} part of the function. This 
individual creates and evaluates another $S_{n-1}$ solution with probability at 
least $\gamma/2n$ (i.e., with probability $1/n$ the 1-bit is flipped and then any 
other bit is flipped with probability 1, and the solution will be 
evaluated with probability $\gamma/2$ after the second bit flip). After at most $\mu 
\cdot O(n)$ generations in expectation the whole population will consist only of $S_{n-1}$  
solutions since this solution is chosen as the parent with probability at least $1/\mu$ in each generation and then it is sufficient to flip the 1-bit in the first mutation step and then any 0-bit in the second mutation step to accept the offspring. Considering that the probability of 
producing two $S_{n-1}$  in one generation is $\binom{\mu}{2}\cdot O(1/n) \cdot 
O(1/n)=O(\log^2n/n^2)$, with probability at least $1-o(1)$ we see at most one 
new $S_{n-1}$ per generation for $o(n^2/ \log^2n)$ generations. Now, we can 
apply Lemma 3 of \cite{OlivetoSudholt2014} to say that in $O(\mu^3 \cdot n)$ 
generations in expectation, the whole population reaches the same age while on 
the local optimum. After at most $\tau$ generations, with probability 
$(1-1/(\mu+1))^{2\mu-1} \cdot 1/(\mu+1)$, $\mu-1$ b-cells die and one b-cell 
survives. 
In the following generation, while $\mu-1$ randomly initialised b-cells are 
added instead of the dead b-cells, the survived b-cell creates an $S_1$ solution 
and evaluates it with probability $(1-O(1/n))(1/e)=\Omega(1)$ for \hypfcm{} by 
flipping all bits and evaluating the last bitstring. For \hypnofcm{} it is 
necessary that only the final solution is evaluated. The expected number of 
evaluations between mutation steps two and $n-1$ is $2\sum_{i=2}^{n/2} 
\gamma/i\leq 1/2$ since $\gamma\leq 1/(4\ln{n})$, and the probability that 
there 
is at least one evaluation is  at most $1/2$ by Markov's inequality. Thus, with 
probability $(1-1/e)(1/2)(1/e)=\Omega(1)$ only the complementary bitstring is 
evaluated and added to the population.  In the following generation, 
this b-cell finds an $S_5$ solution by flipping at most six bits and evaluating it
with probability at least $\gamma/6$. This individual will be added to the population with its age set to zero if the complementary bitstring ($S_{n-1}$) is not evaluated (with probability $(1-1/e)$). In the same generation the $S_1$ solution dies with probability $1/2$ due to ageing. 

Next, we show that the $S_5$ solutions will take over the population, and the first 
point of SP will be found before any $S_{n-1}$ is found. 
An $S_5$ creates an $S_{n-5}$ individual and an $S_{n-5}$ individual creates an 
$S_{5}$ individual with constant probability by evaluating complementary 
bitstrings. Thus, it takes $O(1)$ generations until the number of $S_5$ and 
$S_{n-5}$ individuals in the population doubles. Since the increase in the 
total number  of $S_5$ and $S_{n-5}$ increases exponentially in expectation, in 
$O(\log{\mu})=O(\log\log{ n})$ generations the 
population is taken over by them. After the take-over since each $S_{n-5}$ 
solution creates a $S_{5}$ solution with constant probability, in the following 
$O(1)$ generations in expectation each $S_{n-5}$ creates an $S_{5}$ solution 
which have higher fitness value than their parents and replace them in the 
population. Overall, $S_5$ solutions take over the population in $O(\log\log{ 
n})$ generations in expectation. 

For $S_5$, $\textsc{HiddenPath}$ has a gradient towards the $SP$ which favors solutions with more 0-bits in the first 5 bit positions. Every improvement on the gradient 
takes $O(2/\gamma \cdot n^2)$ generations in expectation since it is enough to flip a 
a precise 1-bit and a 0-bit in the worst case.  Considering that there are five different fitness values on 
the gradient, in $O(5 \cdot 2 \cdot n^2/\gamma)=O(n^2/\gamma)$ generations in 
expectation the first 
point of the SP will be found. Applying Markov's inequality, this 
time will not exceed $O(n^{5/2}/\gamma)$ with probability at least 
$1-1/\sqrt{n}$. 

Now we go back to the probability of finding a locally optimal point before finding an 
SP point. Due to the symmetry of the hypermutation operator probability of creating an $S_{n-1}$ solution from an $S_{5}$ solution is identical to create an $S_{n-1}$ solution from and $S_{n-5}$ solutions. The probability of increasing the number of 0-bits by $k$ given that the initial number of 0-bits is $k$, is at most $(2i/n)^k$ due to the Ballot theorem since each improvement reinitialises a new ballot game with higher disadvantage (see the proof of Theorem~\ref{th:nfcmneg} for a more detailed argument). Thus, the probability that a local optimal solution is sampled is $O(n^{-4})$. The probability that such an event never happens before finding SP is $1-o(1)$. After finding SP, in $O(n\log 
n)$ generations in expectations the global optimum will be found at the end of 
the SP. The probability of finding any locally optimal point from SP is at most 
$O(1/n^4)$, hence this event would not happen before reaching the global optimum 
with probability $1-o(1)$. Overall, the runtime is dominated by 
$O(\tau+n^{5/2}/\gamma )$ which give us $O((\tau+n^{5/2}/\gamma ) \cdot 
\mu(1+\gamma
\log n))$ as the expected number of fitness evaluations due to Lemma~\ref{lem:wald}. Since 
$\Omega(1/\log{n})=\gamma\leq 1/(2\ln{n})$, the upper bound reduces to $O((\tau+\mu n^{5/2}\log{n})$.\qed
\end{proof}

\begin{restatable}{lemma}{fcmprobound} \label{lem:fcmcprobound}
The probability that a solution with at least $ (n/2) + |a| + b$ 1-bits, where  $-n/2\leq a \leq n/2$ and $b<n/2$, is sampled at mutation step $k$, given the initial  solution has $(n/2)+ a$ 1-bits, is bounded above by $e^{-\frac{b^2}{k}}$.
\end{restatable}

\begin{proof}
%With overwhelmingly high probability any randomly initialised solution has more than $(n/2)- n^{2/3}$ 1-bits. 
%We will use Serfling's theorem, to find the probability that a solution with more than $ (n/2) + |a| + n^{2/3}$ 1-bits  is sampled at mutation step $k$, given the initial  solution has $(n/2)+ a$ 1-bits for $a\geq - n^{2/3}$.  
For the input bitstring  $x$, let the multiset of weights $C:=\{c_i | i\in 
[n]\}$ be defined as $c_i:=(-1)^{x_i}$ (i.e., $c_i=-1$ if $x_i=1$ and  $c_i=1$ 
if $x_i=0$). Thus, for permutation $\pi$ of bit-flips over $[n]$, the number of 
1-bits after the $k$th mutation step is $\textsc{OneMax}(x) + \sum_{j=1}^{k} 
c_{\pi_{j}}$  since flipping the 
position $i$ implies that the number of 1-bits changes by $c_i$. 
Let $\bar{X}:=(1/k)\sum_{j=1}^{k} c_{\pi_{j}}$ be the sample mean, 
$\bar{\mu}:= (1/n)\sum_{j=1}^{n}c_i$ the population mean. The 
number of 1-bits which incur the weight of $-1$ when flipped is at least $n/2 
+a$. Thus, $\bar{\mu}\leq \left(1/n\right) \left(-\left(n/2\right)- 
a+\left(n/2\right)-a\right))=-2a/n$.  
To find a solution with at least $(n/2) + |a| + b$ 1-bits it is necessary that the sample mean $\bar{X}$ is at least  $\left(|a| - a + b\right)/k$. Thus, 	
\begin{align*} 
&\bar{X} \geq \frac{|a| - a + b}{k}   \implies  \bar{X} - \bar{\mu} \geq \frac{|a| - a + b}{k} + \frac{2a}{n} \geq \frac{b}{k}+ \frac{|a| -a }{k} + \frac{2a}{n} \geq \frac{b}{k}\\
%
%
 %(\bar{X}-   \bar{\mu}) = (\bar{X}+2a)  \geq  |a| - a + n^{2/3}+2a = |a| + a + n^{2/3} \geq n^{2/3}  \implies \\
&\implies \sqrt{k}(\bar{X}-   \bar{\mu}) \geq  \frac{b}{\sqrt{k} }.
\end{align*}
According to Theorem~\ref{thm:serfling}, 
\[Pr\left\{\sqrt{k}(\bar{X}-\bar{\mu}) \geq 
 \frac{b}{\sqrt{k} }\right\} \leq 
\text{exp}\left(-\frac{2 \left(\frac{b}{\sqrt{k}}\right)^2}{\left( 1-\frac{k-1}{n}\right)\left( 1- (-1)\right) } \right) 
\leq  \text{exp}\left(-\frac{b^2}{k}\right).
\]
\qed
\end{proof}

\fcmcliff*

\begin{proof}
Since each bit value in a solution initialised uniformly at random is equal to 
$1$ with probability $1/2$, the number of 1-bits in 
any initial solution is between $(n/2)- n^{3/5}$ and $(n/2)+ n^{3/5}$ with 
overwhelmingly high probability due to Chernoff bounds.
1-bits. According to Lemma~\ref{lem:fcmcprobound}, the probability that an offspring with more than $a+n/10$ 1-bits is mutated from a parent with  $a \geq (n/2)- n^{3/5}$ 1-bits is in the order of $2^{\Omega(n)}$ using the union bound. Therefore with overwhelmingly high probability we will not observe that a solution with less than $n-d-n/10$ 1-bits having an offspring with more than $n-d$ 1-bits. However, solutions with $n-d\geq\ b \geq n-d-n/10$ 1-bits, have higher fitness than post-cliff solutions with less than $b+d$ 1-bits. Thus, according to Lemma~\ref{lem:fcmcprobound} the probability that an acceptable solution is obtained is in the order of $2^{-\Omega(d)}= 2^{-\Omega(n)}$ again using union bound over $n$ mutation steps.
\end{proof}

\nfcmpos*

\begin{proof}
Since $\gamma=1/(n \log^2 n) $, the expected number of fitness function 
evaluations   per iteration $O(1+\gamma\log{n})$ (see Lemma~\ref{lem:wald}),
is in the order of $\Theta(1)$. On the first \textsc{OneMax} slope, the 
algorithm improves by the first bit flip with probability at least 
$(n-d)/n=\Theta(1)$ and then evaluates this 
solution with probability $p_1=1/e=\Theta(1)$. This implies that the local 
optimum
is found in $O(n)$ fitness evaluations in expectation after  initialisation. 

A solution at the local optimum can only improve by finding the unique 
globally optimum
solution, which requires the hypermutation to flip precisely $d$ 0-bits 
in the first $d$
mutation steps which occurs with probability $\binom{n}{d}^{-1}$. We 
pessimistically assume that
this direct jump never happens and assume that once a solution at the local 
optimum is added to 
the population, it reaches age $\tau$ in some iteration $t_0$. We consider the 
chain of events 
that starts at $t_0$ by \textbf{1)} the addition of a  solution with $(n-d+1)$ 1-bits 
locally optimal to the population (with probability $(1/e) \cdot (n-d)/n$), \textbf{2)} 
the deletion of the locally optimal 
solution due to ageing with probability $1/2$, \textbf{3)} the survival of the post-cliff 
solution with probability $1/2$ and in iteration $t_0 +1$, \textbf{4)} improvement of the 
post-cliff solution fitness function by hypermutation, which happens with a 
constant probability and effectively resets the new solution's age to zero. If 
all of these four events occur consecutively (which happens with constant 
probability), the algorithm can start climbing the second \textsc{OneMax} slope 
with local  moves (i.e., by considering only the first mutation steps) which are evaluated 
with constant probability. Then, the \textsc{Cliff} function is optimised in 
$O(n \log{n})$ function evaluations like \textsc{OneMax} unless a pre-cliff 
solution replaces the current individual. The rest of our analysis will focus on 
the probability that a pre-cliff solution is sampled and evaluated given that 
the algorithm has a  post-cliff solution with age zero at iteration $t_0 +1$. 

If the current solution is a post-cliff solution, then the final bitstring 
sampled by the hypermutation has a worse fitness level than the current 
individual. The probability that \hypnofcm{} evaluates at least one solution between 
mutation steps two and $n-1$ (event $\mathcal{E}_{nv}$), is upper bounded by 
$\sum\limits_{n-1}^{i=2} \gamma/i< 2\gamma\cdot \log{n}=2/(n\log{n})$. We 
consider the $O(n\log{n})$ generations until a post-cliff solution with age zero 
reaches the global optimum. The probability that event $\mathcal{E}_{nv}$ never 
occurs in any iteration until the optimum is found is at least 
$(1-2/(n\log{n}))^{O(n\log{n})}= e^{-O(1)}=\Omega(1)$, a constant probability. 
Thus, every time we create a post-cliff solution with age zero, there is at 
least a constant probability that the global optimum is reached before any 
solution that is not sampled at the first or the last mutation step gets 
evaluated. The first mutation step cannot yield a pre-cliff solution, and the 
last mutation step cannot yield a solution with better fitness value. Thus, with 
a constant 
probability the post-cliff solution finds the optimum. If it fails to do so 
(i.e., a pre-cliff solution takes over as the current solution or a necessary 
event does not occur at iteration $t_0+1$), then in at most $O(n \log {n})$ 
iterations another chance to create a post-cliff solution comes up
and the process is repeated. In expectation, a constant number of trials will be 
necessary until the optimum is found and since each trial takes $O(n \log {n})$ 
fitness function evaluations, thus our claim follows.
\qed
\end{proof}

\nfcmneg*
\begin{proof}
Consider \fastoptia{} with a current solution having more than $n-d$ (i.e., 
post-cliff) and less than $n-d+2\sqrt{n}$ 1-bits. We will show that with 
overwhelmingly high probability, \hypnofcm{} will yield a solution with less than 
$n-d$ (i.e., pre-cliff) and more than $n-2d+2\sqrt{n}$ 1-bits before the 
initial individual is mutated into a solution with more than $n-d+2\sqrt{n}$ 
1-bits. This observation will imply that a pre-cliff solution with better 
fitness will replace the post-cliff solution before the post-cliff solution is 
mutated into a globally optimal solution. We will then show that it is also 
exponentially unlikely that any pre-cliff solution mutates into a 
solution with more than $n-d+\sqrt{n}$ 1-bits to complete our proof.

We will now provide a lower bound on the probability that \hypnofcm{} with 
post-cliff input solution $x$ yields a pre-cliff solution with higher fitness 
value than $x$. 

We will start by determining the earliest mutation step $r_{min}$, that a 
pre-cliff solution with worse fitness than $x$ can be sampled. For any 
post-cliff solution 
$x$, $\textsc{Cliff}_d(x)=\textsc{OneMax}(x)-d+(1/2)$, and any pre-cliff 
solution $y$ with $\textsc{OneMax}(x)-d+1$ 1-bits has a higher fitness than 
$x$. We obtain the rough bound of $r_{min}\geq d- 2\sqrt{n}$ by considering the 
worst-case event  that \hypnofcm{} picks $d$ 1-bits to flip consecutively. 

Let $\ell(x)$ denote the number of extra 1-bits a post-cliff solution has in 
comparison to a locally optimal solution (i.e, 
$\textsc{OneMax}(x)=n-d+\ell(x)$).
Next, we will use Serfling's bound to show that with a constant probability 
\hypnofcm{} will find a pre-cliff solution before $3 \ell(x)$ mutation steps and it 
will keep sampling pre-cliff solutions until $r_{min}$. 

 For the input bitstring of \hypnofcm, $x$, let the multiset of weights 
$C:=\{c_i | i\in [n]\}$ be defined as $c_i:=(-1)^{x_i}$ (i.e., $c_i=-1$ if 
$x_i=1$ and  $c_i=1$ if $x_i=0$). Thus, for a permutation $\pi$ of bit-flips over 
$[n]$, the number of 1-bits after the $k$th mutation step is 
$\textsc{OneMax}(x) + \sum_{j=1}^{k} c_{\pi_{j}}$  since flipping the 
position $i$ implies that the number of 1-bits changes by $c_i$. 

Let $\bar{\mu}:= (1/n)\sum_{j=1}^{n}c_i$ be the population mean of 
$C$ and $\bar{X}:=(1/3\ell 
(x))\sum_{j=1}^{3 \ell (x)} c_{\pi_{j}}$ the sample mean. Since the 
$\textsc{Cliff}$ parameter $d$ is less than $n/4$,   $$\bar{\mu}\leq (1/n) 
\left(\left(-3n/4\right)+ 
\left(n/4\right)\right)=-1/2$$. 
In order to have a solution with at least $n-d+1$ 1-bits at mutation step $3 
\ell (x)$, the following must hold:
\begin{align*}
   &3\ell (x) \bar{X} \geq -\ell(x) \iff  \bar{X} \geq 
-\frac{1}{3}\\  &\implies  \bar{X} - \bar{\mu} \geq 
-\frac{1}{3}+\frac{1}{2} =\frac{1}{6} \iff \sqrt{3 \ell(x)}\left(\bar{X} - 
\bar{\mu}\right) \geq \frac{\sqrt{3 \ell(x)}}{6}.
\end{align*}

% 
%  For $k \geq n\frac{a}{2a-c}$,
% \begin{align*}
%  & \bar{X} - \bar{\mu} \geq -\frac{a n}{k} + 2a \implies 
%  \bar{X} - \bar{\mu} \geq -\frac{an}{n\frac{a}{2a-c}} + 2a =c  \\
% &\iff\\
% &\sqrt{n\frac{a}{2a-c}}\left( \bar{X} - \bar{\mu}\right) \geq 
% \sqrt{n\frac{a}{2a-c}} c.  
% \end{align*}

The probability that a pre-cliff solution will not be found in mutation step $3\ell(x)$ 
 follows from Theorem \ref{thm:serfling}, with sample mean 
$\bar{X}$, population mean $\bar{\mu}$, sample size $3\ell(x)$, population 
size $n$, $c_{min}=-1$ and $c_{max}=1$.
\begin{align*}
Pr\left\{\sqrt{3 \ell(x)}\left(\bar{X} - 
\bar{\mu}\right) \geq \frac{\sqrt{3 \ell(x)}}{6}\right\}  
 &\leq \text{exp}\left(- \frac{2 \left(\frac{\sqrt{3 \ell(x)}}{6}\right)^2}{\left(1-\left( 
\frac{3 \ell(x) -1}{n}    
\right)\right)(1-(-1))^2}\right)\\
&\leq e^{-\Omega\left(\ell(x)\right)}.
\end{align*}

Thus, with probability $(1-e^{-\Omega\left(\ell(x)\right)})$, we will sample 
the first pre-cliff solution after $3\ell(x)$ mutation steps. We focus our 
attention on post-cliff solutions with $1\leq \ell(x) \leq 2\sqrt{n}$ and can 
conclude that for such solutions the above probability is in the order of 
$\Omega(1)$.
Since the number of 0-bits changes by one at every mutation step, the event of 
finding a solution with at most $n-d$ bits implies that at some point a solution 
with exactly $n-d$ 1-bits has been sampled. Let $k_0\leq 3 \ell(x)$ be the 
mutation step where  a locally optimum solution is found for the first time. Due 
to the Ballot theorem the probability that a solution with more than $n-d$ 1-bits is sampled
after $k_0$ is at most $2d/n \leq 1/2$. So, with probability at least $1/2$, the 
\hypnofcm{} will keep sampling pre-cliff solutions until $r_{min}\leq 
d-2\sqrt{n}=\Omega(n)$.
We will now consider the probability that at least one of the solutions sampled 
between $k_0$ and $r_{min}$ is evaluated. Since the evaluation decisions are 
taken independently from each other the probability that none 
of the solutions are evaluated is 
\begin{align*}
\prod\limits_{i=k_{0}}^{r_{min}}\left( 1 - \frac{\gamma }{
i}\right) \leq  \prod\limits_{i=3\ell(x)}^{r_{min}}\left( 1 - \frac{\gamma 
}{
i}\right) \leq \prod\limits_{i=6 \sqrt{n}}^{r_{min}}\left( 1 - \frac{\gamma 
}{
i}\right)   \leq \prod\limits_{i=6 \sqrt{n}}^{r_{min}}\left( 1 - 
\frac{1}{(c_1 \log{n})i }\right)   
\end{align*}
for some constant $c_1$ since $\gamma = \Omega(1/\log{n})$. We will 
separate this product into $\lfloor \log{(r_{min}/6\sqrt{n})} \rfloor$ smaller 
products and show that each smaller product can be bounded from above by 
$e^{-1/(2c_1\log)}$. The first subset contains the factors with 
indices $i \in \{(r_{min}/2)+1,\ldots, r_{min} \}$, the second set  $i \in 
\{(r_{min}/4)+1,\ldots, r_{min}/2 \}$ and $j$th set (for any $j \in [\lfloor 
\log{(r_{min}/6\sqrt{n})} \rfloor]$)  $ i \in \{ r_{min} 2^{-j}+1,\ldots, 
r_{min} 2^{-j+1}\}$. If some indices are not covered by these sets due to the 
floor operator,  we will ignore them since they can only make the final product 
smaller. Note that we assume any logarithm's base is two 
unless it is specified otherwise. 
\begin{align*}
% \prod\limits_{i=6 \sqrt{n}}^{r_{min}}\left( 1 - 
% \frac{1}{(c_1 \log{n})i }\right) \leq 
&\prod\limits_{j=1}^{\lfloor \log{(r_{min}/6\sqrt{n})} 
\rfloor}\prod\limits_{i=r_{min} 2^{-j}+1 }^{r_{min}2^{-j+1}}\left( 1 - 
\frac{1}{(c_1 \log{n})i }\right)  \\&\leq \prod\limits_{j=1}^{\lfloor 
\log{(r_{min}/6\sqrt{n})} \rfloor}\left( 1 - \frac{2^{j-1}}{(c_1 
\log{n})r_{min} }\right)^{r_{min} 2^{-j}} \leq 
\prod\limits_{j=1}^{\lfloor \log{(r_{min}/6\sqrt{n})} 
\rfloor}e^{-1/(2c_1\log{n})}\\ &\leq e^{-\lfloor \log{(r_{min}/6\sqrt{n})} 
\rfloor/(2c_1\log{n})} = e^{-\Omega(1)}  
\end{align*}
where in the second line we made use of the inequality $(1-c n^{-1})^{n}\leq 
e^{-c}$ and in the final line our previous observation that $r_{min}=\Omega(n)$.
This implies that at least one of the sampled pre-cliff individuals will be 
evaluated at least with constant probability. At this point we have established 
that a pre-cliff solution will be added to the population with constant 
probability if the initial post-cliff solution $x$ has a distance between 
$\sqrt{n}$ and $2\sqrt{n}$ to the local optima. 
 
 Let $\mathcal{E}_{i,k}$ for $k>1$ denote the event that hypermutation samples a 
 solution with $i-k$ 0-bits given that the initial solution has $i$ 0-bits. 
 Since 
 the number of 0-bits change by one at every mutation, $\mathcal{E}_{i,k}$ 
 implies $\mathcal{E}_{i,k-1}$. In particular, for $\mathcal{E}_{i,k}$ to happen 
 first $\mathcal{E}_{i,k-1}$ must happen and then another improvement must be 
 found. Given $\mathcal{E}_{i,k-1}$, the probability that a new improvement is 
 found is less than $2i/n$ because of two reasons. First, the number of 0-bits 
 has decreased with respect to the initial solution and second, there are fewer 
 solutions to be sampled before \hypnofcm terminates. Thus, we can conclude:
$ \prob{\mathcal{E}_{i,k}} \leq \prob{ \mathcal{E}_{i,k-1}} 
 \frac{2i}{n} \leq \left(\frac{2i}{n}\right)^{k}
 $.
This implies that it is 
exponentially unlikely that a pre-cliff solution is mutated into a solution with 
more than $n-d+\sqrt{n}$ 1-bits. Moreover, the probability that post-cliff solutions are
improved by more than $n^{1/6}$ is less than $4^{-n/6}$, which implies that with 
overwhelmingly high probability it takes at least $n^{(1/2)-(1/6)}=n^{1/3}$ iterations before
a solution $x$ with $\ell(x)<\sqrt{n}$ is mutated into a solution $x'$ with 
$\ell(x')>2\sqrt{n}$. Since we established that a pre-cliff solution is evaluated with constant
probability at each iteration, we can conclude that at least one such individual is sampled
in $n^{1/3}$ iterations with overwhelmingly high probability. 
Since the \fastoptia{} cannot follow the post-cliff gradient to the optima with overwhelmingly high
probability it relies on making the jump from local optima to global optima. Given an initial solution with $y\in \{(n/3),\ldots, n-d+2\sqrt{n}\}$ 1-bits, the probability of jumping
to the unique global optimum is $2^{\Omega(-n)}$ as well, thus our claim follows.
\qed
\end{proof}

%\fcmprobound*

\end{document}